\def\year{2021}\relax
\documentclass[letterpaper]{article} 
\usepackage{aaai21_arxiv}  
\usepackage{times}  
\usepackage{helvet} 
\usepackage{courier}  
\usepackage[hyphens]{url}  
\usepackage{graphicx} 
\usepackage{natbib}  
\usepackage{caption} 
\nocopyright

\usepackage{amsmath, amsthm, amsfonts, amssymb, caption, graphicx, color, mathtools, svg}
\newtheorem{theorem}{Theorem}

\newtheorem{proposition}[theorem]{Proposition}
\newtheorem{lemma}[theorem]{Lemma}
\newtheorem{corollary}[theorem]{Corollary}

\usepackage[ruled,vlined]{algorithm2e}
\usepackage{algorithmic}
\usepackage[title]{appendix}

\makeatletter
\def\thmhead@plain#1#2#3{%
  \thmname{#1}\thmnumber{\@ifnotempty{#1}{ }\@upn{#2}}%
  \thmnote{ {\the\thm@notefont#3}}}
\let\thmhead\thmhead@plain
\makeatother

\newlength\myindent
\newcommand{\err}{\mathrm{err}}

\title{Communication-Aware Collaborative Learning }
\title{Communication-Aware Collaborative Learning}
\author {
        Avrim Blum,\textsuperscript{\rm 1}
        Shelby Heinecke,\textsuperscript{\rm 2}
       Lev Reyzin\textsuperscript{\rm 3} \\
}
\affiliations {
    \textsuperscript{\rm 1} Toyota Technological Institute at Chicago \\
    \textsuperscript{\rm 2} Salesforce Research \\
    \textsuperscript{\rm 3} University of Illinois at Chicago \\
    avrim@ttic.edu, shelby.heinecke@salesforce.com, lreyzin@uic.edu
}

\begin{document}

\maketitle

\begin{abstract}
Algorithms for noiseless collaborative PAC learning have been analyzed and optimized in recent years with respect to sample complexity. In this paper, we study collaborative PAC learning with the goal of reducing communication cost at essentially no penalty to the sample complexity. We develop communication efficient collaborative PAC learning algorithms using distributed boosting. We then consider the communication cost of collaborative learning in the presence of classification noise. As an intermediate step, we show how collaborative PAC learning algorithms can be adapted to handle classification noise. With this insight, we develop communication efficient algorithms for collaborative PAC learning robust to classification noise.
\end{abstract}

\section{Introduction}

Collaborative learning was recently formalized by Blum et al.~\shortcite{BlumHPQ17} as a PAC learning model. In this collaborative PAC setting, there is a domain $X$, over which are $k$ distributions, referred to as \emph{players}.
There is also a center node that orchestrates the learning process. The goal of collaborative PAC learning is to learn classifiers from data provided by the players that generalize well on each of players' distributions simultaneously. Note that this is distinct from the related distributed learning setting, where the goal is to learn classifiers that generalize well on the mixture of players' distributions \cite{BalcanBFM12}. 

There are generally a few styles of collaborative PAC learning. In the \emph{personalized learning setting}, which is the
main focus of our paper,  the goal is to learn a classifier for each player with generalization error less than $\epsilon$, with probability $1-\delta$. Another setting is the {centralized learning setting}, where the goal is learn a single classifier with generalization error less than $\epsilon$ on each players' distribution with probability $1-\delta$. The efficiency of a collaborative learning algorithm is assessed by its \emph{overhead}, defined as the ratio of the sample complexity of learning in the collaborative setting to the sample complexity of learning in the single player setting. An overhead of at least $k$ indicates that the collaborative learning algorithm offers no sample complexity benefit over individual PAC learning. An overhead less than $k$ indicates that the collaborative algorithm is more sample efficient than individual PAC learning. Collaborative PAC learning algorithms have been optimized in subsequent works with respect to overhead, and hence sample complexity \cite{BlumHPQ17, ChenZZ18, NguyenZ18, Qiao18}.

Certain difficulties may arise in real-world applications of collaborative PAC learning. First, communicating data between players and the center can be costly. Second, the data from players may be noisy. Consider the example described in \cite{BlumHPQ17} where $k$ players represent hospitals serving different demographics of the population. In this network of hospitals, each of which generates an abundance of data, transmitting data to the center is costly and thus hospitals want to minimize the amount of data transmitted. Additionally, mistakes may be present in the labels of the data at the hospitals, due to clerical errors and misdiagnoses, among other reasons. Given access to only the noisy data from the hospitals, we wish to learn classifiers that generalize well with respect to each hospital's underlying noiseless distribution. We tackle both difficulties in this paper. First, we develop communication-aware collaborative learning algorithms in the noiseless setting that enjoy reduced communication costs at no penalty to the sample complexity. Then, we develop communication-aware collaborative learning algorithms in the presence of classification noise, where each player has label noise rate $\eta_i < \frac{1}{2}$. 

The algorithms and analysis in this work focus on personalized learning. We discuss the applications of our insights and analyses to the centralized learning setting in the Appendix. Omitted proofs are also included in the Appendix.

\section{Previous work}
Algorithms for collaborative PAC learning have been analyzed and optimized in \cite{BlumHPQ17, ChenZZ18, NguyenZ18, Qiao18} with respect to sample complexity. The collaborative PAC framework was formalized in \cite{BlumHPQ17}, where they also develop an optimal algorithm in the personalized setting with $O(\ln(k))$ overhead and a suboptimal algorithm in the centralized setting with $O(\ln^2(k))$ overhead. We recall their algorithm, which we refer to as Personalized Learning (Algorithm \ref{personalized-learning}), and the corresponding sample complexity result below.
\begin{algorithm}
\caption{Personalized Learning  \cite{BlumHPQ17}}
\label{personalized-learning}
\SetKwFunction{proc}{TEST}
\SetAlgoLined

\KwIn{$H$, $k$ distributions $D_i \sim X$, $\delta' = \delta/{2 \log(k)}$, $\epsilon > 0$}
\KwOut{$f_1, ..., f_k \in H$}
Let $N_1= \{1,...,k\}$\;
\For{$j= 1, ..., \lceil \log(k) \rceil$} {
Draw sample $S$ of size $m_{\epsilon/4,\delta'}$ from mixture $D_{N_j} = \frac{1}{|N_j|} \sum_{i \in N_j} D_i$\;
Select consistent hypothesis $h_j \in H$ on $S$\;
$G_j \leftarrow$ \proc{$h_j, N_j, \epsilon, \delta'$}\;
$N_{j+1} = N_j \setminus G_j$\;
\For{$i \in G_j$}{
$f_i \leftarrow h_j$\;
}
}
\KwRet{$f_1, ..., f_k$}\\
{}
\setcounter{AlgoLine}{0}
\SetKwProg{myproc}{Procedure}{}{}
\myproc{\proc{$h, N, \epsilon, \delta$}}{
\For{$i \in N$}{
Draw $T_i = O\left(\frac{\ln(\frac{|N|}{\epsilon \delta})}{\epsilon}\right)$ samples from $D_i$\;
}
\KwRet{$\{i \mid \text{err}_{T_i}(h) \leq \frac{3\epsilon}{4}\}$}
}
\end{algorithm}
\begin{theorem}[{\hspace{1sp}\cite{BlumHPQ17}}]\label{personalized-samples} For any $\epsilon, \delta >0$, and hypothesis class $H$ of finite VC-dimension $d$, the sample complexity of Personalized Learning is 
\[
m = O\left(\frac{\ln(k)}{\epsilon}\left((d+k)\ln\left(\frac{1}{\epsilon}\right) + k \ln\left(\frac{k}{\delta}\right) \right)\right).
\]
When $k\ln(k) = O(d)$, the sample complexity is $\tilde{O}(\log(k)\frac{d}{\epsilon})$.
\end{theorem}

Personalized Learning yields an exponential improvement in the sample complexity with respect to the baseline; it improves the baseline's linear $k$ dependence to logarithmic dependence, a drastic improvement for settings with a large number of players. 

Subsequent works \cite{ChenZZ18, NguyenZ18} improve upon their suboptimal centralized learning algorithm  using multiplicative weights approaches. In contrast to these works, we focus on the communication complexity of personalized and centralized learning. We build on the structure of these previously developed algorithms to obtain both sample and communication efficiency in our algorithms. Additionally, we consider communication-aware collaborative learning in the presence of classification noise. The previous work of \cite{Qiao18} considers the collaborative PAC learning where some fraction of players behave truthfully while the remaining players behave adversarially. In addition to considering a different noise model than our work, \cite{Qiao18} show that centralized learning is impossible in their setting and they do not consider communication complexity. To the best of our knowledge, no previous work has addressed the communication complexity of collaborative PAC learning.

\section{Background}
We now define notation and key concepts used in this paper. Let $X$ denote the instance space and $Y = \{0,1\}$ denote the set of possible labels. Let $H$ denote a hypothesis class with finite VC-dimension $d$. We will assume the setting of realizable PAC learning, hence the target hypothesis $h^*$ is in the hypothesis class $H$. The sample complexity of collaborative learning algorithms is defined in the standard way. The focus of this paper is on the communication cost of collaborative learning. We define the communication cost as the total number of samples transmitted between players in the execution of collaborative learning algorithms. To compute communication costs accurately and consistently, we carefully outline the implementation assumptions of our collaborative learning model. First, we define the completion of an algorithm as when each player is in possession of a classifier that has generalization error less than $\epsilon$. Second, we assume that each player has computing power and a priori access to the hypothesis class $H$, $\epsilon, \delta$, and $k$. Third, we assume the broadcast model of communication, also known as the shared blackboard model, in which all players can observe all samples and bits transmitted to the center.

The second half of this paper considers collaborative learning in the presence of classification noise. In this setting, each player has their own distribution $D_i \sim X$ and their own classification noise rate $\eta_i < \frac{1}{2}$. Each player can generate instance-label pairs $(x, y)$, where $x \sim D_i$, and with probability $1-\eta_i$, $y = h^*(x)$, or with probability $\eta_i$, $y= \lnot h^*(x)$. We let $ $EX$_{\eta_i}(\cdot)$ denote the noisy distribution induced by a player's instance-label generating process. The center node orchestrating the learning process has full knowledge of players' noise rates but is not aware of the players' distributions. 

The collaborative PAC learning criteria in the presence of noise is the same as in noiseless collaborative PAC learning except that the learned classifiers must generalize well on each individual player's \emph{clean distribution}, that is, their distribution $D_i$ without label noise. For $h \in H$, let $err_{T}($EX$_{\eta_i}(\cdot), h)$ denote the empirical error of concept $h$ on $T$ points generated from $ $EX$_{\eta_i}(\cdot)$. The definition of empirical error is standard and defined as
\[
err_{T}(\text{EX}_{\eta_i}(\cdot), h) = \frac{1}{T} \sum_{j = 1}^T \mathbf{1}_{\text{EX}_{\eta_i} (x_j) \neq h(x_j)}.
\]
There are two types of generalization errors of $h$ to consider. The first is the error of $h$ on the \emph{noisy distribution}, that is, the distribution $D_i$ in the presence of label noise. The second is the error of $h$ on the underlying clean distribution. In the classification noise setting, the learner only has access to samples from the noisy distribution, but the goal of learning is to generalize well with respect to the clean distribution. With access only to the noisy distribution, we use the generalization error with respect to the noisy distribution as a stepping stone in our analysis. The generalization error on the noisy data distribution, $ $EX$_{\eta_i}(\cdot)$, is defined as
\[
\err_{D_i}(\text{EX}_{\eta_i}(\cdot), h) = \mathbb{E}_{T \sim D_i^T}[\err_{T}(\text{EX}_{\eta_i}(\cdot), h)].
\]
The generalization error on the clean data distribution, $D_i$, is denoted $\err_{D_i}(h)$, and defined as follows,
\[
\err_{D_i}(h) = \mathbb{E}_{T \sim D_i^T}[\err_{T}(h)] = \Pr_{x \sim D_i}[h(x) \neq h^*(x)].
\]

In our algorithms and analysis, we use the classic empirical risk minimization (ERM) approach and sample complexity result of PAC learning with classification noise, in the single-player setting, recalled below. 

\begin{theorem}[\cite{AngluinL87, Laird88}]\label{angluin-laird}
Let $H$ denote a hypothesis class with finite VC-dimension $d$. Let $D$ be a distribution on $X$ and $\eta_i < \frac{1}{2}$. Let $\text{EX}_{\eta_i}(\cdot)$ denote an oracle that returns $(x,h^*(x))$ with probability $1-\eta_i$ or $(x,\lnot h^*(x))$ with probability $\eta_i$. Given any sample $S$ drawn from $\text{EX}_{\eta}$, an algorithm $A$ that produces a hypothesis $h \in H$ that minimizes disagreements with $S$ satisfies the PAC criterion, i.e. for any $\epsilon, \delta >0$ and any distribution $D$ on $X$, $\Pr_{S \sim D^m}[err_{D}(h) \geq \epsilon] \leq \delta$, with sample complexity $$m_{\epsilon, \delta, \eta_i} =O \left( \frac{d \log\left({1}/{\delta}\right)}{\epsilon (1-2\eta_i)^2}\right).$$
\end{theorem}

We note that since the confidence parameter $\delta$ is handled in a standard fashion, for the duration of this paper we suppress $\delta$ dependency for clarity.

\section{Communication-aware personalized learning}

We define the communication cost of a collaborative PAC learning algorithm as the total number of samples transmitted to the center. In contrast, the sample complexity reflects the total number of samples, whether transmitted or not, consumed by the algorithm. Our goal is to achieve communication efficiency, while retaining sample efficiency, in the personalized learning setting. In this section, we develop a personalized learning algorithm whose sample complexity matches that of Personalized Learning (Theorem \ref{personalized-samples}) and whose communication cost is less than that of Personalized Learning, deeming our algorithm the best of both worlds. 

Before describing our approach, we first compute the communication costs of the baseline approach and Personalized Learning. The personalized learning baseline approach is where each player draws $\tilde{O(}\frac{d}{\epsilon})$ examples locally from their own distributions and independently learns their own classifier. This baseline requires no communication to the center. Hence, simultaneous communication and sample efficiency is necessary for our algorithms to be meaningful as we are competing with a baseline whose communication complexity is zero. In other words, if both sample complexity and communication cost are a concern, it will only make sense to choose algorithms other than this baseline if the communication cost is not too much and the sample complexity is substantially lower.

The communication cost of Personalized Learning was not considered in previous works. We compute the communication complexity of Personalized Learning, in light of our implementation assumptions, in the following proposition.

\begin{proposition} The communication cost of Personalized Learning is
\[
\tilde{O}\left(\log(k)\frac{d}{\epsilon} \right)
\]
samples plus $\tilde{O}\left(k\log\left(\frac{d}{\epsilon}\right)\right)$ additional bits of communication. 
\end{proposition}
\begin{proof} We describe the implementation details for Personalized Learning, described in Algorithm \ref{personalized-learning}. Consider round $j$. In the first step, the center computes the number of samples to request from each player by drawing $m_{\epsilon/4, \delta'}/|N_j|$ samples from the uniform multinomial distribution. The center communicates this quantity to each player, costing $O(k\log(\frac{d}{\epsilon}))$ bits.
The players then communicate their requested quantity of samples. By assumption of the broadcast model, each player can see the samples transmitted by other players so all players can learn a consistent hypothesis locally, costing no communication in this step. After learning the consistent hypothesis $h_j$, each player implements \texttt{TEST} locally, costing no communication. Afterwards, they communicate a single bit to the center indicating whether or not \texttt{TEST} passed with $h_j$, costing $O(k)$ bits of communication. Therefore, the total communication over $\log(k)$ rounds is $\tilde{O}(\log(k)\frac{d}{\epsilon})$ samples plus additional $O(k\log(k)\log(\frac{d}{\epsilon})) = \tilde{O}(k\log(\frac{d}{\epsilon}))$ bits of communication. 
\end{proof}

Table \ref{summary} summarizes the sample and communication complexities of the baseline approach, Personalized Learning, and our algorithm, which we call Personalized Learning using Boosting. While our results state that there will be additional bits communicated to orchestrate these algorithms, they are not included in the tables as we are chiefly concerned with the number of samples communicated, as their representations can grow for large $d$. For completeness, we provide the full table, including additional bits communicated, in the Appendix.

The primary driver of communication inefficiency in Personalized Learning is the error parameter, $\epsilon$. In applications such as the hospital scenario described in the introduction, learning highly accurate classifiers is crucial, hence $\epsilon$ is expected to be extremely small. Therefore, our goal is to improve communication complexity exponentially with respect to $\epsilon$, while retaining the logarithmic $k$ dependence granted by Personalized Learning. In particular, we show that our communication-efficient personalized learning algorithm has $O(\log(\frac{1}{\epsilon}))$ dependence in communication complexity.

\begin{table*}[t]
\centering
\caption{Sample and Communication Costs of Personalized Learning Variants}
\label{summary}
\def\arraystretch{1.5}
\begin{tabular}{l | c | c |}
\cline{2-3}
\multicolumn{1}{c|}{\textbf{}} & \multicolumn{1}{c|}{\textbf{Sample Complexity}} & \multicolumn{1}{c|}{\textbf{Samples Communicated}} \\ \cline{2-3} 
Baseline & $\tilde{O}(k\frac{d}{\epsilon})$ & $\tilde{O}(1)$ \\ \cline{2-3} 
Personalized Learning& $\tilde{O}(\log(k)\frac{d}{\epsilon})$ & $\tilde{O}(\log(k)\frac{d}{\epsilon})$ \\ \cline{2-3} 
Personalized Learning using Boosting & $\tilde{O}(\log(k)\frac{d}{\epsilon})$ & $\tilde{O}(\log(k)d\log(\frac{1}{\epsilon}))$ \\ 
\cline{2-3} 
\end{tabular}
\end{table*}

Our approach to improving communication cost is to replace the first step in Personalized Learning with Distributed Boosting \cite{BalcanBFM12}, while keeping the remaining Personalized Learning algorithm intact. Distributed Boosting is a distributed implementation of AdaBoost (\cite{FreundS97}) that learns a consistent hypothesis in $\tilde{O}(\log(\frac{1}{\epsilon}))$ rounds. We note that the objective of Distributed Boosting is to learn a classifier with error less than $\epsilon$ on the \emph{mixture} of distributions. We recall the communication complexity of Distributed Boosting below.

\begin{theorem}[\hspace{1sp}\cite{BalcanBFM12}] \label{distributed-boosting-result} Any class $H$ of finite VC-dimension $d$ can be learned to error $\epsilon$ in $\tilde{O}(\log(\frac{1}{\epsilon}))$ rounds and $O(d)$ examples plus $O(k \log(d))$ bits of communication per round using the distributed boosting algorithm.
\end{theorem}

By using Distributed Boosting as the first step in Personalized Learning, by Theorem \ref{distributed-boosting-result} we will achieve logarithmic dependence on $\epsilon$ in communication cost. However, we must be careful that we don't achieve this improved communication at the cost of higher sample complexity. To the best of our knowledge, the sample complexity of Distributed Boosting was not previously analyzed. We derive the sample complexity of Distributed Boosting in the next section, showing that Distributed Boosting can be implemented with the same sample complexity as AdaBoost.

\subsection{Sample complexity of distributed boosting}

We first recall the sample complexity of AdaBoost \cite{FreundS97}. In AdaBoost, a large sample, denoted by $S$, is drawn from an unknown distribution. Throughout AdaBoost, $S$ is perpetually resampled. The size of $S$, the size of the \emph{reservoir} of points used in the AdaBoost routine, is the sample cost. To review the sample complexity of AdaBoost, we first recall the VC-dimension of the hypothesis class $H$ after $T$ rounds of boosting.

\begin{lemma}[\hspace{1sp}\cite{FreundS97}] \label{boosting-rounds} Suppose the weak learner in AdaBoost learns a classifier with constant error in each round. Then, $\tilde{O}(\ln(\frac{1}{\epsilon}))$ rounds of AdaBoost are needed to learn a classifier with zero training error.
\end{lemma}

Let $d_{\text{boost}}$ denote the VC-dimension of the hypothesis class after $T$ rounds of boosting.

\begin{lemma}[\hspace{1sp}\cite{FreundS97}] \label{lemma-vc-dim-boosting} Let $H$ denote the base class of hypotheses with VC-dimension $d$. After $T$ rounds of boosting, the resulting hypothesis class has VC-dimension $d_{\text{boost}} = O(dT\log(T)) = \tilde{O}(dT)$.
\end{lemma}

We recall the folklore result of the sample complexity of AdaBoost, which follows immediately from Lemma \ref{boosting-rounds}, Lemma \ref{lemma-vc-dim-boosting}, and realizable PAC sample complexity bounds.

\begin{lemma}\label{adaboost-samples} The sample complexity of AdaBoost is
\[
m_{\text{boost}} = O\left(\frac{d_{\text{boost}}}{\epsilon}\right) = \tilde{O}\left(\frac{d}{\epsilon}\right).
\]
\end{lemma}

We now show that the sample complexity of Distributed Boosting is the same as that of AdaBoost. In Distributed Boosting, there are $k$ players implementing AdaBoost. Each player has a reservoir of points, $S_i$, from which the center resamples. The sample cost is the sum over the players' sample reservoirs, $\sum_{i=1}^k S_i$. From standard learning theory we know lower and upper bounds on $\sum_{i=1}^k |S_i|$, but the size with which to initialize each individual $S_i$ so that the algorithm is correct was not previously analyzed. During each round of Distributed Boosting, the number of samples the center requests from a player can increase and in the original analysis of Distributed Boosting, $S_i$ was defined to be ambiguously large \cite{BalcanBFM12}. We give clarity to the size of each $S_i$ needed for Distributed Boosting. To do so, we first propose adding the following one-time preprocessing step to Distributed Boosting: let the center draw $m_{\text{boost}}$ points from a uniform multinomial distribution to determine the sample size of each player's reservoir $S_i$. Communicating these sample sizes to the players cost $\tilde{O}(k\log(\frac{d}{\epsilon}))$ bits in total. This preprocessing step adds only a negligible cost to the bits communicated in Distributed Boosting. By initializing each reservoir in this way, we limit the total sample size to $\sum_{i=1}^k |S_i| = m_{\text{boost}}$. 
%

\begin{proposition}\label{sample-complexity-distributed-boosting} The sample complexity of Distributed Boosting is
\[
O\left(\frac{d_{\text{boost}}}{\epsilon}\right) = \tilde{O}\left(\frac{d}{\epsilon}\right).
\]
\end{proposition}

\begin{proof} The derivation of the sample complexity of Distributed Boosting follows from the fact that Distributed Boosting is equivalent to AdaBoost with a single player and sample size $S = \cup_{i=1}^k S_i$. In this case, we know the sample complexity is $\tilde{O}(\frac{d}{\epsilon})$ by Lemma \ref{adaboost-samples}. By adding the preprocessing step described above, we restrict the sample complexity of the algorithm to $m_{\text{boost}}$. We now show that with $m_{\text{boost}}$ samples, across players as prescribed by the preprocessing step, Distributed Boosting remains correct. The pre-sampling step where the center draws from a multinomial distribution to determine the number of samples to request from each player remains unaffected by the new preprocessing step. However, in the sampling phase, it is possible that the center requests more points from a player than the player possesses. In this case, the player simply samples from their reservoir $S_i$ i.i.d. proportional to the weights corresponding to the points. The weak learning step of Distributed Boosting is unaffected by the preprocessing step since it is still receiving a sample drawn i.i.d. from the boosting-weighted mixture of players. And finally, we note that the sample weights updating step remains unaffected. Therefore, $\tilde{O}(\frac{d}{\epsilon})$ samples suffice for Distributed Boosting and adding the preprocessing step to Distributed Boosting achieves the sample complexity.
\end{proof}

The result above reveals an important fact about the sample complexity of Distributed Boosting with $k$ players -- the sample complexity is surprisingly not dependent on $k$. Therefore, using Distributed Boosting, we can achieve sample and communication efficiency for the personalized learning setting, which we formalize in the next section.

\subsection{Communication-efficient personalized learning }

Our approach to achieving a communication and sample efficient algorithm for the personalized learning setting is to replace the first step of Personalized Learning with the Distributed Boosting algorithm while leaving the remaining steps of Personalized Learning intact. We refer to our approach as Personalized Learning using Boosting. Using the sample complexity result on Distributed Boosting from the previous section, we compute the sample complexity of Personalized Learning using Boosting, showing that it is indeed equal (up to polylogarithmic terms) to the optimal sample complexity achieved by Personalized Learning.

\begin{theorem} The sample complexity of Personalized Learning using Boosting is 
\[
\tilde{O}\left(\log(k)\frac{d}{\epsilon}\right)
\]
when $k\ln(k) = O(d)$.
\end{theorem}

We now formally compute the communication complexity of Personalized Learning using Boosting, showing that it is an exponential improvement over the communication complexity of Personalized Learning with respect to $\epsilon$.

\begin{theorem} The communication complexity of Personalized Learning using Boosting is
\[
\tilde{O}\left(\log(k)\left(d\log\left(\frac{1}{\epsilon}\right)\right)\right)
\]
samples plus an additional $\tilde{O}(k\log(d)\log(\frac{1}{\epsilon}))$ bits of communication.
\end{theorem}
\begin{proof} We consider a single round of our algorithm. The communication complexity of the first step is given in Theorem \ref{distributed-boosting-result} as $\tilde{O}(d\log(\frac{1}{\epsilon}))$ examples plus $\tilde{O}(k \log(d)\log(\frac{1}{\epsilon}))$ bits of communication. Recall that each step in distributed boosting, all players learn the same weak learning classifiers locally. Therefore, when the distributed boosting algorithm completes, each player has all $\log(k)$ weak classifiers and can therefore sum them to create the final boosting classifier $h_j$, costing no communication. Using the boosting classifier in the \texttt{TEST} step, there is no communication needed as the players simply need to test the boosting classifier on $T_j$ samples drawn from their own distributions. The players each send one bit of communication to the center indicating if they passed \texttt{TEST} or not, costing $O(k)$ bits for all $k$ players. Therefore the total communication complexity over $\log(k)$ rounds is $\tilde{O}(\log(k)(d\log(\frac{1}{\epsilon}))$ samples plus $\tilde{O}(k \log(d)\log(\frac{1}{\epsilon})) + O(k) = \tilde{O}(k \log(d)\log(\frac{1}{\epsilon}))$ additional bits of communication.
\end{proof}

\section{Communication-aware personalized learning with classification noise}
Thus far we have studied the communication cost of collaborative PAC learning without any assumptions of noise in the data. However, the presence of noise in data is often unavoidable in real-world learning scenarios. For instance, in the case where $k$ hospitals work collaboratively to learn a diagnosis classifier, it is possible that a hospital's data has label noise from clerical errors or misdiagnoses. In this section, we consider communication-aware collaborative learning in the presence of classification noise, where each player has their own label noise rate $\eta_i < 1/2$ so that for any data point $x$ drawn from their distribution $D_i$, with probability $\eta_i$ they produce the wrong label and with probability $1-\eta_i$ they produce the correct label. We note that collaborative PAC learning in the presence of classification noise has not been previously analyzed. Thus, to build communication-efficient collaborative learning algorithms robust to classification noise, we first must analyze how to adapt collaborative learning to handle classification noise more generally. We note that the analysis and approaches in this section can be applied similarly to centralized learning in the presence of classification noise. Due to space constraints, we defer the details of centralized learning to the Appendix.

\subsection{Personalized learning with classification noise}
Consider the baseline approach to personalized learning with classification noise, where the center  requests $m_{\epsilon, \delta, \eta_i}$ samples from each player and learns an empirical risk minimizer (ERM), following exactly as in standard single-player PAC learning with classification noise (Theorem \ref{angluin-laird}). In this case, the sample complexity is $\sum_{i=1}^k m_{\epsilon, \delta, \eta_i} = O(km_{\epsilon, \delta, \eta_{\text{MAX}}})$. The goal then is to develop a personalized learning algorithm with improved sample complexity.
%

We present our algorithm, Personalized Learning with Classification Noise (Algorithm \ref{personalizedCN}), and show that it indeed improves upon the sample complexity of the baseline. The skeleton of our algorithm models that of (noiseless) Personalized Learning, but we make adjustments to handle classification noise. In the first and second steps of Personalized Learning, the center draws $m_{\epsilon/4, \delta'}$ samples and learns a consistent hypothesis. In contrast, in our algorithm, the center draws $m_{\epsilon/4, \delta', \bar{\eta}_{N_j}}$ points in total from the uniform mixture of players and learns an ERM hypothesis. When learning in the presence of classification noise, the existence of a hypothesis in the hypothesis class consistent with a sample generated from a noisy distribution is not guaranteed. Hence, our algorithm finds an ERM hypothesis instead of a consistent hypothesis. By Theorem \ref{angluin-laird} the ERM hypothesis has error $\epsilon/4$ when trained on $m_{\epsilon/4, \delta', \bar{\eta}_{N_j}}$ samples drawn from the noisy distribution. Finally, our  \texttt{CN-TEST} subroutine differs from the \texttt{TEST} subroutine in Personalized Learning in that ours accounts for the individual noise rates of the players. Essentially, players must draw a factor of $\frac{1}{(1-2\eta_i)}$ more samples in \texttt{CN-TEST} than in \texttt{TEST} and the testing criterion is adjusted to reflect the relationship between drawing from noisy distribution and generalizing on the clean distribution.

\begin{algorithm}[]
\caption{Personalized Learning with Classification Noise}
\label{personalizedCN}
\SetKwFunction{proc}{CN-TEST}
\SetAlgoLined

\KwIn{$H$, $k$ distributions $D_i \sim X$ with error rates $\eta_i <\frac{1}{2}$, $\delta' = \delta/{2 \log(k)}$, $\epsilon > 0$}
\KwOut{$f_1, ..., f_k \in H$}
Let $N_1= \{1,...,k\}$\;
\For{$j= 1, ..., \lceil \log(k) \rceil$} {
Draw sample $S$ of size $m_{\epsilon/4,\delta', \bar{\eta}_{N_j}}$ from mixture $D_{N_j} = \frac{1}{|N_j|} \sum_{i \in N_j} D_i$\;
Select ERM hypothesis $h_j \in H$ on $S$\;
$G_j \leftarrow$ \proc{$h_j, N_j, \epsilon, \delta'$}\;
$N_{j+1} = N_j \setminus G_j$\;
\For{$i \in G_j$}{
$f_i \leftarrow h_j$\;
}
}
\KwRet{$f_1, ..., f_k$}\\
{}
\setcounter{AlgoLine}{0}
\SetKwProg{myproc}{Procedure}{}{}
\myproc{\proc{$h, N, \epsilon, \delta$}}{
\For{$i \in N$}{
Draw $T_i = O\left(\frac{\ln(\frac{|N|}{\delta})}{\epsilon (1-2\eta_i)}\right)$ samples from $D_i$\;
}
\KwRet{$\{ i \mid err_{T_i}($EX$_{\eta_i}, h_j)\leq \eta_i + \frac{3 \epsilon}{4}(1-2\eta_i) \}$}
}
\end{algorithm}

The correctness of our algorithm will largely follow from the correctness results of Personalized Learning shown in \cite{BlumHPQ17}, but with modifications to handle classification noise. We start by showing that even in the presence of classification noise, the first and second steps in our algorithm yield a classifier that performs with error $\epsilon/4$ on the mixture. 

\begin{lemma}\label{ERM-markov-lemma} The ERM $h_j$ learned in Personalized Learning with Classification Noise has error no more than $\frac{\epsilon}{2}$ on at least half of the distributions in $N_j$.
\end{lemma}

Next, we consider the \texttt{CN-TEST} subroutine, which tests if the learned ERM is a good classifier for each of the remaining players with respect to their underlying clean distribution. Recall that we only have access to their noisy data. Our analysis uses the following lemma from \cite{AngluinL87} that connects the generalization error of a concept $h$ on the noisy distribution to the generalization error of $h$ on the underlying clean distribution. 

\begin{lemma}[\hspace{1sp}\cite{AngluinL87}] \label{lemma-angluin-laird}
Let $D$ be a distribution on $X$. Let $\eta_i$ denote the classification noise rate and $h^*\in H$ denote the target function. Then, 
\[
err_{D}(\text{EX}_{\eta_i}(\cdot), h) = \eta_i + err_{D}(h)(1-2\eta_i).
\]
\end{lemma}

The following lemmas regarding the correctness of \texttt{CN-TEST} are due to multiplicative Chernoff bounds and Lemma \ref{lemma-angluin-laird}.

\begin{lemma}\label{lemma-rcn-test-1}
With probability $1-\delta'$, if $h_j$ passes \texttt{CN-TEST} then $err_{D_i}(\text{EX}_{\eta_i}(\cdot), h_j) \leq \eta_i + (1-2\eta_i)\epsilon$. Hence, $err_{D_i}(h_j) \leq \epsilon$.
\end{lemma}

\begin{lemma}
\label{lemma-rcn-test-2}
With probability $1-\delta'$, if $err_{D_i}(\text{EX}_{\eta_i}(\cdot), h_j) \leq \eta_i + (1-2\eta_i)\frac{\epsilon}{2}$, then $h_j$ passes \texttt{CN-TEST}. 
Hence, if $err_{D_i}(h_j) \leq \frac{\epsilon}{2}$, then $h_j$ passes \texttt{CN-TEST}.
\end{lemma}

We combine the above lemmas to show correctness of our algorithm, Personalized Learning with Classification Noise.

\begin{proposition} Personalized Learning with Classification Noise satisfies the personalized collaborative PAC learning criteria.
\end{proposition}

Now, we compute the sample complexity of Personalized Learning with Classification Noise.

\begin{proposition}\label{personalizedCN-bound}
The sample complexity of Personalized Learning with Classification Noise is 
\[
O\left(\log(k) \left(\frac{k \ln(k\log(k))}{\epsilon (1-2\eta_{\text{MAX}})} + \frac{d \ln(\log(k))}{\epsilon (1-2\eta_{\text{MAX}})^2}\right) \right).
\]
When $k \ln(k) = O(d)$, the sample complexity simplifies to $\tilde{O}\left(\log(k)\frac{d}{\epsilon (1-2\eta_{\text{MAX}})^2 }\right)$.
\end{proposition}

Personalized Learning with Classification Noise improves upon the sample complexity of the baseline since it has logarithmic dependence on $k$ instead of linear dependence on $k$. For settings with a large number of players, such as in a network of databases or a network of IoT devices, our algorithm can enjoy improved sample complexity. In fact, simplifying the sample complexity in Proposition \ref{personalizedCN-bound} with respect to constant $\epsilon, \delta$, and $\eta_{\text{MAX}}$, and assuming $k \ln(k) = O(d)$, shows that our algorithm has $\tilde{O}(\log(k))$ overhead compared to the overhead of $\tilde{O}(k)$ from the baseline approach. 

\begin{table*}[!t]
\caption{Sample and Communication Costs of Personalized Learning Variants with Classification Noise}
\centering
\label{summary-CN}
\def\arraystretch{1.5}
\begin{tabular}{l | c | c |}
\cline{2-3}
\multicolumn{1}{c|}{\textbf{}} & \multicolumn{1}{c|}{\textbf{Sample Complexity}} & \multicolumn{1}{c|}{\textbf{Samples Communicated}} \\ \cline{2-3} 
Baseline & $\tilde{O}\left(k\frac{d}{\epsilon (1-2\eta_{\text{MAX}})^2}\right)$ & $\tilde{O}(1)$ \\ \cline{2-3}
Personalized Learning with CN& $\tilde{O}\left(\log(k)\frac{d}{\epsilon (1-2\eta_{\text{MAX}})^2 }\right)$ & $\tilde{O}\left(\log(k)\frac{d}{\epsilon(1-2\eta_{\text{MAX}})^2 }\right)$ \\ \cline{2-3} 
Personalized Learning with CN using Boosting & $\tilde{O}\left(\log(k)\frac{d}{\epsilon (1-2\eta_{\text{MAX}})^2}\right)$ & $\tilde{O}\left(\log(k)d\log\left(\frac{1}{\epsilon (1-2\eta_{\text{MAX}})}\right)\right)$\\ 
\cline{2-3} 
\end{tabular}
\end{table*}

\subsection{Communication-efficient personalized learning with classification noise}
We now return to the main goal of this section, which is to develop a communication-efficient personalized learning algorithm robust to classification noise. We first review the communication-efficient baseline approach, which is when each player draws $m_{\epsilon, \delta, \eta_i}$ samples from their own distribution and learns a classifier locally. The sample complexity of this baseline is $O\left(k\frac{d}{\epsilon (1-2\eta_{\text{MAX}})^2}\right)$ and requires no samples nor bits of communication. To improve communication costs of personalized learning in the presence of noise, we build on Personalized Learning with Classification Noise developed in the previous section. We compute the communication cost of Personalized Learning with Classification Noise below.

\begin{proposition}The communication complexity of Personalized Learning with Classification Noise is
\[
\tilde{O}\left(\log(k)\frac{d}{\epsilon(1-2\eta_{\text{MAX}})^2}\right)
\]
samples and $\tilde{O}\left(k \log\left(\frac{d}{\epsilon(1-2\eta_{\text{MAX}})^2}\right)\right)$ additional bits of communication.
\end{proposition}

\begin{proof} Recall that $\delta' = O(\delta/\log(k))$. In the first step, the center computes the number of samples to request from each player by drawing $m_{\epsilon/4, \delta'}/|N_j|$ samples from the uniform multinomial distribution. The center communicates this quantity to each player, costing $O(k\log(\frac{d}{\epsilon(1-2\bar{\eta}_{N_j})^2}))$ bits. The players then communicate their requested quantity of samples.
Since we are in the broadcast model, each player observes all points communicated to the center, thereby allowing each player to learn an ERM hypothesis locally. Each player then implements \texttt{CN-TEST} locally, costing no communication. After completing \texttt{CN-TEST}, each player must send one bit to the center indicating their pass/fail result of \texttt{CN-TEST}, costing $O(k)$ bits in total. Over all $O(\log(k))$ rounds, the communication complexity is $O\left(\log(k)\frac{d \log(\log(k))}{\epsilon (1-2\bar{\eta}_K)^2}\right) = \tilde{O}\left(\log(k)\frac{d}{\epsilon(1-2\eta_{\text{MAX}})^2}\right)$ with an additional $O\left(\log(k)k\log\left(\frac{d}{\epsilon(1-2\bar{\eta}_{N_j})^2}\right)\right) + O(k\log(k)) = \tilde{O}\left(k \log\left(\frac{d}{\epsilon(1-2\eta_{\text{MAX}})^2}\right)\right)$ bits of communication.
\end{proof}

Table \ref{summary-CN} summarizes the sample and communication costs of the baseline approach, Personalized Learning with Classification Noise, and our communication-efficient algorithm, Personalized Learning with Classification Noise using Boosting.

As discussed previously, we focus on the learning scenario where players want to learn highly accurate classifiers. Thus our goal is to develop an algorithm that improves dependence on $\frac{1}{\epsilon (1-2\eta_{\text{MAX}})}$ in samples communicated.

Our algorithm, Personalized Learning with Classification Noise using Boosting, is described as follows. We simply replace the first step of our noise-robust personalized learning algorithm, Personalized Learning with Classification Noise, with Distributed Agnostic Boosting \cite{ChenBC16}, while leaving the rest of Personalized Learning with Classification Noise intact. It is well known that boosting in the presence of classification noise is not straightforward. In fact, it has been shown that boosting the generalization error rate past the noise rate, so that $\eta_{\text{MAX}} > \epsilon$, is hard \cite{KalaiS03}. To avoid these issues, we restrict our attention to boosting the error $\epsilon$ up to the noise rate $\eta_{\text{MAX}}$, so that $\eta_{\text{MAX}} \leq \epsilon$. In this restricted regime, we use Distributed Agnostic Boosting from \cite{ChenBC16}, since classification noise is a special case of agnostic learning. Distributed Agnostic Boosting assumes access to a $\beta$-weak agnostic learner, which returns a hypothesis $h$ so that $err_{D}(h) \leq \min_{h' \in H}err(h') + \beta$ \cite{ChenBC16}. We recall the sample and communication complexities of Distributed Agnostic Boosting below.

\begin{theorem}[\hspace{1sp}\cite{ChenBC16}]\label{dist-agnostic-boosting-samples} Suppose Distributed Agnostic Boosting has access to a $\beta$-weak agnostic learner. Then, the sample complexity is 
\[
\tilde{O}\left(\frac{d}{\epsilon^2(1/2-\beta)^2}\right).
\]
\end{theorem}

It is well known that learning in the presence of classification noise is a special case of agnostic learning. We therefore derive the following corollary to the sample complexity of Distributed Agnostic Boosting in the restricted setting of classification noise. 

\begin{corollary} Suppose Distributed Agnostic Boosting has access to a $\beta$-weak agnostic learner. Let $\beta$ be a fixed constant. The sample complexity of Distributed Agnostic Boosting in the restricted setting of classification noise, where $\eta_{\text{MAX}} \leq \epsilon$, is
\[
\tilde{O}\left(\frac{d}{\epsilon(1-2\eta_{\text{MAX}})^2}\right).
\]
\end{corollary}

We now recall the communication complexity of Distributed Agnostic Boosting.

\begin{theorem}[\hspace{1sp}\cite{ChenBC16}] Suppose Distributed Agnostic Boosting has access to a $\beta$-weak agnostic learner. Then, Distributed Agnostic Boosting achieves error $\frac{2err_{D}(H)}{1/2 -\beta} + \epsilon$ by using at most $O\left(\frac{\log(\frac{1}{\epsilon})}{(1/2 - \beta)^2}\right)$ rounds, each communicating 
$O\left(\frac{d}{\beta}\log\left(\frac{1}{\beta}\right)\right)$ samples and 
$\tilde{O}\left(kd\log^2\left(\frac{d}{(1/2 - \beta)\epsilon}\right)\right)$
words of communication.
\end{theorem}

Similarly, we derive a corollary that holds specifically for the classification noise setting.

\begin{corollary} Suppose Distributed Agnostic Boosting has access to a $\beta$-weak agnostic learner. Let $\beta$ be a fixed constant. The communication complexity of Distributed Agnostic Boosting in the restricted setting of classification noise, where $\eta_{\text{MAX}} \leq \epsilon$, consists of $O\left(\log\left(\frac{1}{\epsilon(1-2\eta_{\text{MAX}})}\right)\right)$ rounds, each communicating $O(d)$ samples and $\tilde{O}\left(kd\log^3\left(\frac{1}{\epsilon (1-2\eta_{\text{MAX}})}\right)\right)$ bits of communication.
\end{corollary}

We derive the following sample and communication complexities of our algorithm, Personalized Learning with Classification Noise using Boosting.

\begin{theorem} The sample complexity of Personalized Learning with Classification Noise using Boosting is 
\[
\tilde{O}\left(\log(k)\frac{d}{\epsilon (1-2\eta_{\text{MAX}})^2}\right).
\]
\end{theorem}

\begin{theorem} The communication complexity of Personalized Learning with Classification Noise using Boosting is
\[
\tilde{O}\left(\log(k)d\log\left(\frac{1}{\epsilon(1-2\eta_{\text{MAX}})}\right)\right)
\]
plus $\tilde{O}\left(kd\log^4\left(\frac{1}{\epsilon (1-2\eta_{\text{MAX}})}\right)\right)$ bits of communication.
\end{theorem}

\clearpage
\section{Acknowledgements}
This work was supported in part by the National Science Foundation under grants CCF-1815011, CCF-1934915, and CCF-1848966. This work was done while Shelby Heinecke was a student at UIC. 
\bibliography{manuscript}

\onecolumn
\appendix

\section{Appendix A: Communication-Aware Personalized Learning}

We first confirm that the new preprocessing step of sampling from a multinomial distribution to determine reservoir sizes adds a negligible number of bits communicated to Distributed Boosting. From Theorem \ref{distributed-boosting-result}, the bits communicated in Distributed Boosting is $\tilde{O}\left(k\log(d)\log\left(\frac{1}{\epsilon}\right)\right)$. Now, we add an additional one-time preprocessing step, costing $\tilde{O}(k\log(\frac{d}{\epsilon}))$ bits. Summing still gives $\tilde{O}\left(k\log(d)\log\left(\frac{1}{\epsilon}\right)\right)$, so the preprocessing step is indeed negligible.

In the table below, we replicate Table \ref{summary} from the paper and include the bits communicated for completeness.
We note that while our approach saves on the number of samples communicated from the Baseline and Personalized Learning, the number of bits communicated increases. However, the number of bits communicated is only slightly more than that of Personalized Learning, and we retain the same savings in sample complexity. We view this as an acceptable tradeoff because we are chiefly concerned with the samples communicated, which is the dominant factor in high dimensions.  We leave improving the bits communicated for future work.

\begin{table}[h!]
\caption{Sample and Communication Costs of Personalized Learning Variants (Including Bits Communicated)}
\label{summary-extended}
\def\arraystretch{1.5}
\resizebox{\columnwidth}{!}{
\begin{tabular}{l | c | c | c |}
\cline{2-4}
\multicolumn{1}{c|}{\textbf{}} & \multicolumn{1}{c|}{\textbf{Sample Complexity}} & \multicolumn{1}{c|}{\textbf{Samples Communicated}} & \textbf{Bits Communicated} \\ \cline{2-4} 
Baseline & $\tilde{O}(k\frac{d}{\epsilon})$ & $\tilde{O}(1)$ & $\tilde{O}(1)$ \\ \cline{2-4} 
Person. Learn. & $\tilde{O}(\log(k)\frac{d}{\epsilon})$ & $\tilde{O}(\log(k)\frac{d}{\epsilon})$ & $\tilde{O}(k \log(\frac{d}{\epsilon}))$ \\ \cline{2-4} 
Person. Learn. using Boosting & $\tilde{O}(\log(k)\frac{d}{\epsilon})$ & $\tilde{O}(\log(k)d\log(\frac{1}{\epsilon}))$ & $\tilde{O}(k\log(d)\log(\frac{1}{\epsilon}))$  \\ 
\cline{2-4} 
\end{tabular}
}
\end{table}

\subsubsection*{Proof of Theorem 9}
\begin{proof} By Proposition 8, the sample complexity of Distributed Boosting is $\tilde{O}(\frac{d}{\epsilon})$. This step is implemented at most $O(\log(k))$ times, totaling
$\tilde{O}(\log(k)\frac{d}{\epsilon})$ samples. The remaining steps of Personalized Learning using Boosting are identical to Personalized Learning. From \cite{BlumHPQ17}, the \texttt{TEST} step uses $O(\log(k)\frac{k}{\epsilon}\ln(\frac{k\log(k)}{\epsilon})) = \tilde{O}(\log(k)\frac{k}{\epsilon})$ samples. Since $k\ln(k) = O(d)$, we have $\tilde{O}(\log(k)\frac{d}{\epsilon})$.
\end{proof}

\section*{Appendix B: Communication-Aware Personalized Learning with Classification Noise}
\subsubsection*{Proof of Lemma 11}
\begin{proof}
First, note that the mixture $D_{N_j}$ has expected error rate $\bar{\eta}_{N_j}$. By Theorem 2, the ERM $h_j$ trained on $m_{\epsilon/4, \delta', \bar{\eta_{N_j}}}$ samples drawn from $D_{N_j}$ has $err_{D_{N_j}}(h_j) \leq \frac{\epsilon}{4}$. As shown in \cite{BlumHPQ17}, Markov's inequality yields the result,
\[
\Pr\left[err_{D_{N_j}}(h_j) \leq 2 \left( \frac{\epsilon}{4} \right)\right] \geq \frac{1}{2}.
\]
\end{proof}

\subsubsection*{Proof of Lemma 12}
\begin{proof} This proof is due to \cite{AngluinL87} but we include it here for completeness. There are two ways in which $h$ can disagree with $\text{EX}_{\eta_i}(\cdot)$ on a point $x \in X$:
(1) $\text{EX}_{\eta_i}(\cdot)$ labels $x$ correctly with probability $1-\eta_i$ and $h$ disagrees with $h^*$, or  
(2) $\text{EX}_{\eta_i}(\cdot)$ labels $x$ incorrectly with probability $\eta_i$ and $h$ agrees with $h^*$. 
The probability that either of these two events occurs is $(1-\eta_i)err_{D}(h)+\eta_i(1-err_{D}(h))= \eta_i + err_{D}(h)(1-2\eta_i)$.
\end{proof}

\subsubsection*{Proof of Lemma 13}
\begin{proof}
As in \cite{BlumHPQ17, ChenZZ18, NguyenZ18}, we use multiplicative Chernoff bounds to prove the performance of \texttt{CN-TEST}. Our multiplicative Chernoff bounds are adjusted to handle each player's noise rate $\eta_i$. By Lemma 12, we scale the generalization error on the noisy distribution to the generalization error on clean distribution. 

Assume that $err_{D_i}(\text{EX}_{\eta_i}(\cdot), h_j) \geq \eta_i + (1-2\eta_i)\epsilon$ and that $T_j \geq \frac{32}{\epsilon(1-2\eta_i)}\ln\left(\frac{|N|}{\delta'}\right)$. Let $P = err_{D_i}(\text{EX}_{\eta_i}(\cdot), h_j) - \eta_i$. We use the Chernoff bound on the random variable $err_{T_j}(\text{EX}_{\eta_i}(\cdot), h_j)$, the empirical error of $h_j$ of $T_j$ samples. If $h_j$ passes \texttt{CN-TEST} for player $i$, then we have the following inequality:
\[
\Pr\left[err_{T_j}(\text{EX}_{\eta_i}(\cdot), h_j) \leq \eta_i + \frac{3\epsilon}{4}(1-2\eta_i)\right] = \Pr \left[err_{T_j}(\text{EX}_{\eta_i}(\cdot), h_j) - \eta_i \leq \frac{3\epsilon}{4}(1-2\eta_i)\right] 
\]
Computing the expected value of $err_{T_j}(\text{EX}_{\eta_i}(\cdot), h_j)$, we have
\begin{align}
\mathbb{E}_{T_j \sim D_i^{T_j}}[err_{T_j}(\text{EX}_{\eta_i}(\cdot), h_j) - \eta_i] &= \mathbb{E}_{T_j \sim D_i^{T_j}}[err_{T_j}(\text{EX}_{\eta_i}(\cdot), h_j)] - \eta_i \\
& = err_{D_i}(\text{EX}_{\eta_i}(\cdot), h_j) - \eta_i \\
&\geq (1-2\eta_i)\epsilon.
\end{align}
Applying the Chernoff bound gives: 
\begin{align}
\Pr\left[err_{T_j}(\text{EX}_{\eta_i}(\cdot), h_j) \leq \eta_i + \frac{3\epsilon}{4}(1-2\eta_i)\right] &= \Pr \left[err_{T_j}(\text{EX}_{\eta_i}(\cdot), h_j) - \eta_i \leq \frac{3\epsilon}{4}(1-2\eta_i)\right] \\
& \leq \Pr \left[err_{T_j}(\text{EX}_{\eta_i}(\cdot), h_j) - \eta_i \leq (1-\frac{1}{4})P\right] \\
&\leq \exp \left(-\frac{1}{2}\left(\frac{1}{4}\right)^2(1-2\eta_i)\epsilon T_j \right) \\
& \leq \frac{\delta'}{|N|} 
\end{align}
By the union bound over $|N|$ players, if $err_{D_i}(\text{EX}_{\eta_i}(\cdot), h_j) \geq \eta_i + (1-2\eta_i)\epsilon$ then $h_j$ passes \texttt{CN-TEST} with probability at most $\delta'$. Hence, \texttt{CN-TEST} is correct with probability $1-\delta'$. By Lemma 12, this implies $err_{D_i}(h_j) \leq \epsilon$.
\end{proof}

\subsubsection*{Proof of Lemma 14} 

We need to following lemma for the proof of Lemma 14.
\begin{lemma} Let $s = \frac{\epsilon(1-2\eta_i)}{4(err_{D_i}(\text{EX}_{n_i}(\cdot), h_j) -\eta_i)}$. Suppose $err_{D_i}(\text{EX}_{\eta_i}(\cdot), h_j) - \eta_i \leq \frac{\epsilon}{2}(1-2\eta_i)$. Then,
\[
(1+s)(err_{D_i}(\text{EX}_{\eta_i}(\cdot), h_j) - \eta_i) \leq \frac{3\epsilon}{4}(1-2\eta_i).
\]
\end{lemma}
\begin{proof} The proof follows as in \cite{NguyenZ18}, but casted to our classification noise setting. Suppose the claim is true, $\frac{3 \epsilon}{4}(1-2\eta_i) \geq (1+s)(err_{D_i}(\text{EX}_{\eta_i}(\cdot), h_j)$. Then,
\begin{align}
\frac{3 \epsilon (1-2\eta_i)}{4 (err_{D_i}(\text{EX}_{\eta_i}(\cdot), h_j) - \eta_i))} &\geq 1+s \\
\frac{3 \epsilon (1-2\eta_i)}{4 (err_{D_i}(\text{EX}_{\eta_i}(\cdot), h_j) - \eta_i))} &\geq 1+ \frac{(1-2\eta_i)\epsilon}{4(err_{D_i}(h_j, \text{EX}_{n_i}) -\eta_i)} \\
\frac{\epsilon (1-2\eta_i)}{2 (err_{D_i}(\text{EX}_{\eta_i}(\cdot), h_j) - \eta_i)} &\geq 1
\end{align}
The last inequality is true since by assumption $err_{D_i}(\text{EX}_{\eta_i}(\cdot), h_j) - \eta_i \leq(1-2\eta_i)\frac{\epsilon}{2} \implies 2(err_{D_i}(\text{EX}_{\eta_i}(\cdot), h_j) - \eta_i) \leq \epsilon(1-2\eta_i)$.
\end{proof}

\begin{proof}
As in \cite{NguyenZ18}, we invoke our classification noise analogue, Lemma 24, and use the multiplicative Chernoff bounds. Let $s$ be defined as in Lemma 24. We consider two cases, (1) when $err_{D_i}(\text{EX}_{\eta_i}(\cdot), h_j) - \eta_i \geq (1-2\eta_i)\frac{\epsilon}{4}$ and (2) when $err_{D_i}(\text{EX}_{\eta_i}(\cdot), h_j) - \eta_i \leq (1-2\eta_i)\frac{\epsilon}{4}$. First, suppose $err_{D_i}(\text{EX}_{\eta_i}(\cdot), h_j) - \eta_i \geq (1-2\eta_i)\frac{\epsilon}{4}$. Let $P = err_{D_i}(\text{EX}_{\eta_i}(\cdot), h_j)-\eta_i$. Then, $s <1$ and, 
\[
\Pr\left[err_{T_j}(\text{EX}_{\eta_i}(\cdot), h_j) - \eta_i \geq (1-2\eta_i)\frac{3\epsilon}{4}\right] \leq \Pr\left[err_{T_j}(\text{EX}_{\eta_i}(\cdot), h_j) - \eta_i \geq (1-2\eta_i)(1+s)\right]. 
\]
By multiplicative Chernoff bounds and Lemma 24,
\begin{align}
\Pr \left[err_{T_j}(\text{EX}_{\eta_i}(\cdot), h_j) - \eta_i \geq (1-2\eta_i)\frac{3\epsilon}{4}\right] &= \Pr\left[err_{T_j}(\text{EX}_{\eta_i}(\cdot), h_j) - \eta_i \geq 3(1-2\eta_i)\frac{\epsilon}{4}\right] \\
&\leq \Pr\left[err_{T_j}(\text{EX}_{\eta_i}(\cdot), h_j) - \eta_i \geq (1+s)(1-2\eta_i)\frac{\epsilon}{4}\right] \\
&\leq \exp\left(-\frac{1}{3}\left(\frac{(1-2\eta_i)\epsilon}{4(err_{D_i}(\text{EX}_{n_i}(\cdot), h_j) -\eta_i)}\right)^2 P T_j\right) \\
&\leq \exp\left(-\frac{1}{12}(1-2\eta_i) \epsilon T_j\right) \\
&\leq \frac{\delta'}{|N|}
\end{align}
when $T_j = O\left(\frac{\ln(\frac{|N|}{\delta'})}{\epsilon (1-2\eta_i)}\right)$. The second case follows by a symmetric argument. In both cases, by the union bound over all $|N|$ players, $h_j$ fails \texttt{CN-TEST} with probability less than $\delta'$. Since $err_{D_i}(h_j, \text{EX}_{\eta_i}) \leq \eta_i + (1-2\eta_i)\frac{\epsilon}{2}$, Lemma 12 implies $err_{D_i}(h_j) \leq \frac{\epsilon}{2}$.
\end{proof}

\subsubsection*{Proof of Proposition 15}
\begin{proof} Given access to noisy data, our algorithm must learn a classifier for each player with generalization error less than $\epsilon$ on each player's underlying clean distribution, with probability $1-\delta$. Our proof follows as in \cite{BlumHPQ17} but with our noise-adapted lemmas. The first and second step in our algorithm yield a classifier that performs with error $\epsilon/4$ on the mixture of distributions by Lemma 11. By Lemmas 13 and 14, $h_j$ exhibits the same performance in \texttt{CN-TEST} as does the consistent hypothesis in \texttt{TEST} in Personalized Learning. The rest of the proof of correctness follows directly from \cite{BlumHPQ17}, where they prove that in each round, at least half of the players are removed with probability $1-\frac{\delta}{\log(k)}$. Therefore, after at most $\log(k)$ rounds, each player is assigned a hypothesis with generalization error less than $\epsilon$, with respect to their clean distribution, with probability $1-\delta$.
\end{proof}

\subsubsection*{Proof of Proposition 16}
\begin{proof}
Recall that $\delta' = O(\frac{\delta}{\log(k)})$. Our algorithm implements \texttt{CN-TEST} for $\log(k)$ rounds on at most $|N| = k$ players, using
\[
O\left(\log(k) \frac{k\ln(\frac{k \log(k)}{\delta})}{\epsilon (1-2\eta_{\text{MAX}})}\right)
\]
samples. Let $K = \{1,...,k\}$. The algorithm learns at most $\log(k)$ ERMs via Theorem 2, using 
\[
O(\log(k) m_{\epsilon/4, \delta', \bar{\eta}_K}) = O\left(\log(k) \frac{d \ln(\frac{\log(k)}{\delta})}{\epsilon (1-2\bar{\eta}_K)^2}\right)
\] 
samples. Note that $\bar{\eta}_K \leq \eta_{\text{MAX}}$. Summing gives the sample complexity result. 
\end{proof}

As in Table \ref{summary-extended} above, here we replicate Table \ref{summary-CN} from the paper and include the bits communicated.  We observe the same tradeoff here as in the noiseless case.

\begin{table}[h!]
\caption{Sample and Communication Costs of Personalized Learning Variants with Classification Noise (Including Bits Communicated)}
\label{summary-CN-extended}
\def\arraystretch{1.5}
\resizebox{\columnwidth}{!}{
\begin{tabular}{l | c | c | c |}
\cline{2-4}
\multicolumn{1}{c|}{\textbf{}} & \multicolumn{1}{c|}{\textbf{Sample Complexity}} & \multicolumn{1}{c|}{\textbf{Samples Communicated}} & \textbf{Bits Communicated} \\ \cline{2-4} 
Baseline & $\tilde{O}\left(k\frac{d}{\epsilon (1-2\eta_{\text{MAX}})^2}\right)$ & $\tilde{O}(1)$ & $\tilde{O}(1)$ \\ \cline{2-4} 
Person. Learn. with CN & $\tilde{O}\left(\log(k)\frac{d}{\epsilon (1-2\eta_{\text{MAX}})^2 }\right)$ & $\tilde{O}\left(\log(k)\frac{d}{\epsilon(1-2\eta_{\text{MAX}})^2 }\right)$ & $\tilde{O}\left(k \log\left(\frac{d}{\epsilon(1-2\eta_{\text{MAX}})^2}\right)\right)$ \\ \cline{2-4} 
Person. Learn. with CN using Boosting & $\tilde{O}\left(\log(k)\frac{d}{\epsilon (1-2\eta_{\text{MAX}})^2}\right)$ & $\tilde{O}\left(\log(k)d\log\left(\frac{1}{\epsilon (1-2\eta_{\text{MAX}})}\right)\right)$ & $\tilde{O}\left(kd\log^4\left(\frac{1}{\epsilon (1-2\eta_{\text{MAX}})}\right)\right)$\\ 
\cline{2-4} 
\end{tabular}
}
\end{table}

\subsubsection*{Proof of Corollary 19}
\begin{proof} Using the standard agnostic learning to classification noise restriction \cite{JabbariHZ12}, the sample complexity for Distributed Agnostic Boosting in this case is $\tilde{O}\left(\frac{d}{\epsilon^2(1-2\eta_{\text{MAX}})^2}\right)$. In the special case of classification noise, we can reduce $\epsilon^2$ dependence to $\epsilon$ dependence following the careful argument in \cite{Laird88}.
\end{proof}

\subsubsection*{Proof of Corollary 21}
Follows immediately from the standard agnostic learning restriction for classification noise \cite{JabbariHZ12} as in Corollary 19. Note that the communication complexity result in Theorem 20 is partially in terms of \emph{words} communicated. We convert words to bits so that we can clearly compare the performances of the algorithms we discussed. To convert from words to bits, and to ensure an approximation within poly($\frac{d}{\epsilon(1-2\eta_{\text{MAX}})}$), we assume that each element in any vector communicated consists of at most $O\left(\log\left(\frac{d}{\epsilon(1-2\eta_{\text{MAX}})}\right)\right)$ bits. 

\subsubsection*{Proof of Theorem 22}
Recall $\delta' = \delta/2\log(k)$. In the first step, Distributed Agnostic Boosting identifies a hypothesis costing $\tilde{O}\left(\frac{d}{\epsilon(1-2\eta_{\text{MAX}})^2}\right)$ samples by Corollary 19. In particular, the identified hypothesis is an $\epsilon$-good hypothesis on the underlying clean mixture of distributions in our setting. After the hypothesis is identified, each remaining player runs implements $\texttt{CN-TEST}$, costing $O\left(\frac{k\ln(\frac{k \log(k)}{\delta})}{\epsilon (1-2\eta_{\text{MAX}})}\right)$ samples. This is repeated over $O(\log(k))$ rounds. Together this gives a sample complexity of
\[
O \left( \log(k) \frac{d}{\epsilon(1-2\eta_{\text{MAX}})^2} + \frac{k \ln (\frac{k \log(k)}{\delta})}{\epsilon (1-2\eta_{\text{MAX}})}\right).
\]
Fixing $\delta$ to a constant and letting $k \ln(k) = O(d)$, we have sample complexity
\[
\tilde{O}\left(\frac{d}{\epsilon(1-2\eta_{\text{MAX}})^2}\right).
\]

\subsubsection*{Proof of Theorem 23}
In the first step, the communication cost of Distributed Agnostic Boosting is $O\left(d\log\left(\frac{1}{\epsilon(1-2\eta_{\text{MAX}})}\right)\right)$ samples by Corollary 21. Also, $\tilde{O}\left(kd\log^3\left(\frac{1}{\epsilon (1-2\eta_{\text{MAX}})}\right)\right)$ additional bits per round, for a total of $\tilde{O}\left(kd\log^4\left(\frac{1}{\epsilon (1-2\eta_{\text{MAX}})}\right)\right)$ bits communicated. The $\texttt{CN-TEST}$ step can be done locally for each player so no communication is necessary, except for transmitting a single bit indicating if the classifier passed $\texttt{CN-TEST}$ or not, costing $O(k)$ bits. This is repeated for at most $O(\log(k))$ rounds, giving the communication cost of
\[
O\left(d\log(k)\log\left(\frac{1}{\epsilon(1-2\eta_{\text{MAX}})}\right)\right)
\]
samples. The total number of bits communicated is 
\[
\tilde{O}\left(kd\log^4\left(\frac{1}{\epsilon (1-2\eta_{\text{MAX}})}\right)\right).
\]

\section*{Appendix C: Communication-Aware Centralized Learning}
We derive communication-aware algorithms for the centralized learning setting. The approaches are the same as those used in the personalized learning setting. We first recall the optimal centralized learning algorithm.

\begin{algorithm}
\caption{Centralized Learning \cite{ChenZZ18, NguyenZ18}}
\label{centralized-learning-algo}
\SetKwFunction{proc}{FAST-TEST}
\SetAlgoLined

\KwIn{$H$, $k$ distributions $D_i \sim X$, $\delta' = \delta/4t$, $\epsilon' = \epsilon/6$, $t = 150\lceil \log(\frac{k}{\delta}) \rceil$}
\KwOut{$h = $ \text{MAJORITY}($\{h_i\}_{i = 1}^t$)}
Initialize $w_{i,0} = 1$ for all $i \in [1,k]$\;
\For{$j= 1, ..., t$} {
Draw sample $S$ of $m_{\epsilon'/16, \delta'}$ samples from mixture $D_j = \frac{1}{\sum_{i=1}^k w_{i,j}}\sum_{i=1}^k w_{i,j}D_i$\;
Select consistent hypothesis $h_j \in H$ on $S$\;
$G_j \leftarrow$ \proc{$h_j, N_j, \epsilon, \delta'$}\;
\For{$i = 1, ..., k$}{
$w_{i, j+1} = \begin{cases} 2w_{i,j} &\mbox{if } i \notin G_j \\
w_{i,j} &\mbox{if } i \in G_j
\end{cases}$\;
}
}
\KwRet{$h = $ \mbox{MAJORITY}($\{h_i\}_{i = 1}^t$)} \\
{}
\setcounter{AlgoLine}{0}
\SetKwProg{myproc}{Procedure}{}{}
\myproc{\proc{$h, k, \epsilon, \delta$}}{
\For{$i = 1, ..., k$}{
Draw sample of size $T_i = O(\frac{1}{\epsilon})$ from $D_i$\;
}
\KwRet{$\{i \mid \text{err}_{T_i}(h) \leq \frac{3\epsilon}{4}\}$}
}
\end{algorithm}

\begin{proposition}[\cite{NguyenZ18}]\label{centralized-samples} The sample complexity of Centralized Learning is 
\[
O\left(\frac{1}{\epsilon}\ln\left(\frac{k}{\delta}\right)\left(d\ln\left(\frac{1}{\epsilon}\right) + k + \ln\left(\frac{1}{\delta}\right)\right)\right).
\]
For $\delta$ constant and $k\ln(k) = O(d)$, this simplifies to
\[
\tilde{O}\left(\log(k)\frac{d}{\epsilon}\right).
\]
\end{proposition}

We compute the communication complexity of Centralized Learning. 

\begin{proposition} The communication cost of Centralized Learning is 
\[
\tilde{O}\left(\log(k)\frac{d}{\epsilon}\right)
\]
samples plus $\tilde{O}(k\log(\frac{d}{\epsilon}))$ bits.
\end{proposition}

\begin{proof}
Consider round $j$. In the first step, the center computes the number of samples to request from each player by drawing $m_{\epsilon'/16, \delta'}$ from the weighted multinomial distribution, based on the weights in $D_j = \frac{1}{\sum_{i=1}^k w_{i,j}}\sum_{i=1}^k w_{i,j}D_i$. After sampling from the multinomial the center sends the number of points to request to each player, costing $\tilde{O}(k \ln(\frac{d}{\epsilon}))$ bits in communication. The players then communicate their samples to the center, costing $\tilde{O}(\frac{d}{\epsilon})$ samples in communication. The players learn the same consistent hypothesis and implement \texttt{FAST-TEST} locally. They each send a bit to the center indicating if they passed $\texttt{FAST-TEST}$ costing $O(k)$ bits in communication. Summing over $O(\log(k))$ rounds, we have that the samples communicated is 
\[
\tilde{O}\left(\log(k)\frac{d}{\epsilon}\right)
\]
and the additional bits communicated is $\tilde{O}(k\log(\frac{d}{\epsilon}))$.
\end{proof}

As in communication-aware personalized learning, to achieve communication efficiency, we replace the first step of Centralized Learning with Distributed Boosting \cite{BalcanBFM12} to get a communication-efficient centralized learning algorithm which we call Centralized Learning using Boosting. Notice that in Centralized Learning the first step is drawing samples from a weighted mixture instead of a uniform mixture as in Personalized Learning. This does not present an issue for Distributed Boosting since it extends to handle weighted mixture setting \cite{BalcanBFM12}. We derive the sample complexity and communication complexity below, showing that Centralized Learning using Boosting enjoys an improved $\log(\frac{1}{\epsilon})$ dependence in communication cost.

\begin{theorem} The sample complexity of Centralized Learning using Boosting is 
\[
\tilde{O}\left(\log(k)\frac{d}{\epsilon}\right)
\]
when $k\ln(k) = O(d)$.
\end{theorem}
\begin{proof}
We replaced the first step of Centralized Learning with Distributed Boosting. By Theorem \ref{sample-complexity-distributed-boosting} , this has sample complexity $\tilde{O}(\frac{d}{\epsilon})$. The remaining steps are the same. Therefore, the sample complexity becomes $\tilde{O}(\log(k)\frac{d}{\epsilon})$ when $k \ln(k) = O(d)$.
\end{proof}

\begin{theorem} The communication complexity of Centralized Learning using Boosting is
\[
\tilde{O}\left(\log(k)d\log\left(\frac{1}{\epsilon}\right)\right)
\]
samples plus $\tilde{O}(k\log(d)\log(\frac{1}{\epsilon}))$ bits of communication.
\end{theorem}

\begin{proof}
The communication cost of Distributed Boosting is given in Theorem \ref{distributed-boosting-result}. The remaining steps follow as in Centralized Learning. Adding the communication costs over $O(\log(k))$ rounds gives the result.
\end{proof}

\section*{Appendix D: Communication-Aware Centralized Learning with Classification Noise}

We now consider centralized collaborative learning in the presence of classification noise. As in the personalized learning setting, to make centralized learning robust to classification noise we make similar  modifications to Centralized Learning. As in \cite{ChenZZ18, NguyenZ18}, our centralized learning algorithm does not iteratively remove players from consideration, but rather, re-weights the players' distributions at each round, reminiscent of boosting. Let $\bar{\eta}_{j}$ denote the weighted average of noise rates determined by the distributions weights in round $j$, 
\[
\bar{\eta}_{j} = \frac{1}{\sum_{i=1}^k w_{i,j}}\sum_{i=1}^k w_{i,j}\eta_i.
\]

We refer to our algorithm as Centralized Learning with Classification Noise. In summary, our algorithm differs from the noiseless centralized learning algorithm in two ways. First, we sample $m_{\epsilon'/16, \delta', \bar{\eta}_{j}}$ from the weighted mixture to account for the presence of classification noise. Second, we use \texttt{CN-FAST-TEST}, a modification of \texttt{FAST-TEST}. Aside from the proof of correctness of our adjustments to handle classification noise, the proof of correctness of our algorithm will follow immediately from \cite{ChenZZ18, NguyenZ18}.

\begin{algorithm}
\caption{Centralized Learning with Classification Noise}
\label{centralizedCN}
\SetKwFunction{proc}{CN-FAST-TEST}
\SetAlgoLined

\KwIn{$H$, $k$ distributions $D_i \sim X$, $\delta' = \delta/4t$, $\epsilon' = \epsilon/6$, $t = 150\lceil \log(\frac{k}{\delta}) \rceil$}
\KwOut{$h = $ \text{MAJORITY}($\{h_i\}_{i = 1}^t$)}
Initialize $w_{i,0} = 1$ for all $i \in [1,k]$
\For{$j= 1, ..., t$} {
Draw sample $S$ of $m_{\epsilon'/16, \delta', \bar{\eta}_{j}}$ samples from mixture $D_j = \frac{1}{\sum_{i=1}^k w_{i,j}} \sum_{i=1}^k w_{i,j}D_i$\;
Select ERM hypothesis $h_j \in H$ on $S$\;
$G_j \leftarrow$ \proc{$h_j, N_j, \epsilon', \delta'$}\;
\For{$i = 1, ..., k$}{
$w_{i, j+1} = \begin{cases} 2w_{i,j} &\mbox{if } i \notin G_j \\
w_{i,j} &\mbox{if } i \in G_j
\end{cases}$\;
}
}
\KwRet{$h = $ \mbox{MAJORITY}($\{h_i\}_{i = 1}^t$)}\\
{}
\setcounter{AlgoLine}{0}
\SetKwProg{myproc}{Procedure}{}{}
\myproc{\proc{$h, N, \epsilon, \delta$}}{
\For{$i \in N$}{
Draw sample of size $T_i = O(\frac{1}{\epsilon' (1-2\eta_i)})$ from $D_i$\;
}
\KwRet{$\{ i \mid err_{T_j}(h, \text{EX}_{\eta_i})\leq \eta_i + \frac{3 \epsilon}{4}(1-2\eta_i)$}
}
\end{algorithm}

We start by showing the correctness of Steps 1 and 2. 

\begin{lemma} \label{lemma-centralized-erm} The ERM $h_j$ established in Step 2 of Algorithm \ref{centralizedCN} satisfies $err_{D_j}(h_j) \leq \frac{\epsilon'}{16}$.
\end{lemma}

\begin{proof} The expected error rate of weighted distribution $D_j$ is $\bar{\eta}_{j}$. Therefore, the result follows by Theorem 2.
\end{proof}

We now prove the correctness of \texttt{CN-FAST-TEST}.

\begin{lemma} \label{lemma-centralized-chernoff}(1) With probability at least .99, if $h_j$ passes \texttt{CN-FAST-TEST} then $err_{D_i}(h_j, $ EX$_{\eta_i}) \leq \eta_i + (1-2\eta_i)\epsilon$. Hence, $err_{D_i}(h_j) \leq \epsilon$. (2) With probability at least .99, if $err_{D_i}(h_j, \text{EX}_{\eta_i}) \leq \eta_i + (1-2\eta_i)\frac{\epsilon}{2}$, then $h_j$ passes \texttt{CN-FAST-TEST}. Hence, if $err(h_j) \leq \frac{\epsilon}{2}$ then $h_j$ passes \texttt{CN-FAST-TEST}.
\end{lemma}
\begin{proof} Follows from multiplicative Chernoff bounds.
\end{proof}

As in Centralized Learning, in Centralized Learning with Classification Noise, players are never eliminated from the algorithm. Instead, their weights are increased or decreased according to the performance of $h_j$ on their distributions. Over $t$ rounds, this gives have $t$ classifiers with varying performance on the players. As in Centralized Learning, we take the majority vote of these classifiers as our final classifier. Thus, we need to verify that the majority vote is sufficiently accurate. We use the following claim directly from \cite{ChenZZ18, NguyenZ18} which also holds in our setting with classification noise.

\begin{lemma}[\hspace{1sp}\cite{ChenZZ18, NguyenZ18}] \label{lemma-centralized-players} For each player, the number of classifiers that have error more than $\epsilon'$ is less than $.4t$ with probability $1-\delta$.
\end{lemma}

Therefore, by the above lemma and analysis in \cite{BlumHPQ17, ChenZZ18, NguyenZ18}, with probability $1-\delta$, the error of the final hypothesis $h$ is less than $\epsilon$ for every player's distribution. Combining Lemmas \ref{lemma-centralized-erm}, \ref{lemma-centralized-chernoff}, \ref{lemma-centralized-players} proves the correctness of our algorithm.

We compute the sample complexity of our algorithm.

\begin{theorem}\label{centralized-CN-samples} The sample complexity of Centralized Learning in the Presence of Classification Noise is 
\[
O\left(\log\left(\frac{k}{\delta}\right)\left( \frac{d \log(\frac{\log(k/\delta)}{\delta})}{\epsilon (1-2\eta_{\text{MAX}})^2} + \frac{k}{\epsilon (1-2\eta_{\text{MAX}})} \right)\right).
\]
When $k\ln(k) = O(d)$ and $\delta$ is constant, this simplifies to 
\[
\tilde{O}(\log(k)\frac{d}{\epsilon (1-2\eta_{\text{MAX}})^2}).
\]
\end{theorem}
\begin{proof}
Recall that $\delta'  = O\left(\frac{\delta}{\log(k/\delta)}\right)$. The algorithm runs for a total of $t = O(\log(\frac{k}{\delta}))$ rounds. In each round, an ERM is learned on the mixture of players using 
\[
m_{\frac{\epsilon'}{16}, \delta', \bar{\eta}_{j}} = O\left(\frac{d \log(\frac{\log(k/\delta)}{\delta})}{\epsilon (1-2\bar{\eta}_j)^2}\right)
\]
samples. Then, the algorithm implements \texttt{CN-FAST-TEST} at each round for all $k$ players, costing $O(\frac{k}{\epsilon (1-2\eta_i)})$ samples. Therefore, over all $t$ rounds, the sample complexity is 
\[
O\left(\log\left(\frac{k}{\delta}\right)\left( \frac{d \log(\frac{\log(k/\delta)}{\delta})}{\epsilon (1-2\eta_{\text{MAX}})^2} + \frac{k}{\epsilon (1-2\eta_{\text{MAX}})} \right)\right).
\]
Simplifying by setting $\delta$ to be a constant and $k \ln(k) = O(d)$ yields the result.
\end{proof}

We now replace the first step in Centralized Learning with Classification Noise with Distributed Agnostic Boosting \cite{ChenBC16}. We will refer to this algorithm as Centralized Learning with Classification Noise using Boosting and the following theorems show that the communication cost is improved to logarthmic dependence on $\frac{1}{\epsilon}$.

\begin{theorem} The sample complexity of Centralized Learning with Classification Noise using Boosting is
\[
\tilde{O}(\frac{d}{\epsilon (1-2\eta_{\text{MAX}})^2})
\]
when $k \ln(k) = O(d)$.
\end{theorem}

\begin{proof} We replace the first step of Centralized Learning with Classification Noise with Distributed Agnostic Boosting \cite{ChenBC16}. Theorem \ref{dist-agnostic-boosting-samples} gives the sample complexity of Distributed Agnostic Boosting. The rest of the algorithm is the same as Centralized Learning with Classification Noise. Adding the samples together gives the result.
\end{proof}

\begin{theorem} The communication complexity of Centralized Learning with Classification Noise using Boosting is
\[
\tilde{O}\left(\log(k)d\log\left(\frac{1}{\epsilon(1-2\eta_{\text{MAX}})}\right)\right)
\]
plus $\tilde{O}\left(kd\log^4\left(\frac{1}{\epsilon (1-2\eta_{\text{MAX}})}\right)\right)$ bits of communication.\end{theorem}

\begin{proof}
The communication cost of the first step which is Distributed Agnostic Boosting in the setting of classification noise is given in Corollary 21. Adding to this the communication cost of the remainder of the algorithm and multiplying by $O(\log(k))$ yields the result.
\end{proof}

\end{document}